\documentclass{article}
\usepackage{amsthm,amsfonts,amsmath,hyperref,bbm,soul}
\usepackage{microtype,graphicx,enumerate,tikz,times}
\usepackage[accepted]{icml2015}
\usepackage{comment}

\usepackage[mathscr]{eucal}

\usetikzlibrary{arrows,positioning} 
\tikzset{
    >=stealth',
    punkt/.style={
           circle,
           draw=black, thick,
           minimum height=2.7em,
           inner sep=0pt,
           text centered},
    pil/.style={
           ->,
           thick,
           shorten <=2pt,
           shorten >=2pt,}
}

\newtheorem{theorem}{Theorem}
\newtheorem{lemma}{Lemma}

\newcommand{\indep}{\mathrel{\text{\scalebox{1.07}{$\perp\mkern-10mu\perp$}}}}
 
\DeclareMathOperator*{\argmin}{arg\,min}

\DeclareMathOperator*{\E}{\mathbb{E}\,}
\renewcommand{\P}{\mathcal{P}}
\renewcommand{\H}{\mathcal{H}}
\newcommand{\U}{\mathcal{U}}
\newcommand{\F}{\mathcal{F}}
\newcommand{\M}{\mathscr{M}}

\newcommand{\Z}{\mathcal{Z}}

\newcommand{\N}{\mathcal{N}}
\newcommand{\Ex}{\mathcal{E}}
\newcommand{\sig}{\mathrm{sign}\circ\!}
\newcommand{\Hyp}{\mathcal{H}}
\newcommand{\Pio}{\mathbb{P}}
\newcommand{\R}{\mathbb{R}}
\newcommand{\B}{\mathcal{B}}
\renewcommand{\L}{\mathcal{L}}
\renewcommand{\d}{\mathrm{d}}
\newcommand{\G}{\mathcal{G}}
\newcommand{\Rp}{R_{\varphi}}
\newcommand{\Rpn}{\hat{R}_{\varphi}}
\newcommand{\Rpnt}{\tilde{R}_{\varphi}}
\newcommand{\f}{f^*}
\newcommand{\fn}{\hat{f}_n}
\newcommand{\fnt}{\tilde{f}_n}

\newcommand{\muP}{\mu_k(\P)}

\icmltitlerunning{Towards a Learning Theory of Cause-Effect Inference}
 
\begin{document} 
\twocolumn[
  \icmltitle{Towards a Learning Theory of Cause-Effect Inference}
  \icmlauthor{David Lopez-Paz$^{1,2}$}{david@lopezpaz.org}
  \icmlauthor{Krikamol Muandet$^1$}{krikamol@tuebingen.mpg.de}
  \icmlauthor{Bernhard Sch\"olkopf$^1$}{bs@tuebingen.mpg.de}
  \icmlauthor{Ilya Tolstikhin$^1$}{ilya@tuebingen.mpg.de}
  \icmladdress{$^1$Max-Planck-Institute for Intelligent Systems}
  \vskip -0.1in
  \icmladdress{$^2$University of Cambridge}
  \icmlkeywords{}
  \vskip 0.3in
]

\begin{abstract}
  We pose causal inference as the problem of learning to classify probability
  distributions. In particular, we assume access to a collection
  $\{(S_i,l_i)\}_{i=1}^n$, where each $S_i$ is a sample drawn from the
  probability distribution of $X_i \times Y_i$, and $l_i$ is a binary label
  indicating whether ``$X_i \to Y_i$'' or ``$X_i \leftarrow Y_i$''.  Given
  these data, we build a causal inference rule in two steps.  First, we
  featurize each $S_i$ using the kernel mean embedding associated with some
  characteristic kernel. Second, we train a binary classifier on such
  embeddings to distinguish between causal directions. We present
  generalization bounds showing the statistical consistency and learning rates
  of the proposed approach, and provide a simple implementation that achieves
  state-of-the-art cause-effect inference.  Furthermore, we extend our ideas to
  infer causal relationships between more than two variables.
\end{abstract}

\section{Introduction}\label{sec:intro}
The vast majority of statistical learning algorithms rely on the exploitation
of associations between the variables under study. Given the argument that all
associations arise from underlying causal structures \cite{Reichenbach56:Time},
and that different structures imply different influences between variables, the question
of how to infer and use causal knowledge in learning acquires great importance
\cite{Pearl00,Scholkopf12}.  Traditionally, the most widely used strategy to
infer the causal structure of a system is to perform interventions on some of
its variables, while studying the response of some others.  However, such
interventions are in many situations unethical, expensive, or even impossible
to realize.  Consequently, we often face the need of causal inference purely
from \emph{observational data}. In these scenarios, one suffers, in the absence
of strong assumptions, from the indistinguishability between
latent confounding ($X \leftarrow Z \rightarrow Y$) and direct causation ($X
\to Y$ or $X \leftarrow Y$).  Nevertheless, disregarding the impossibility of
the task, humans continuously learn from experience to accurately
infer causality-revealing patterns.  Inspired by this successful learning, and
in contrast to prior work, this paper addresses causal inference by 
unveiling such causal patterns directly from data. In particular, we assume
access to a set $\{(S_i,l_i)\}_{i=1}^n$, where each $S_i$ is a sample set drawn
from the probability distribution of $X_i \times Y_i$, and $l_i$ is a binary
label indicating whether ``$X_i \to Y_i$'' or ``$X_i \leftarrow Y_i$''.  Using
these data, we build a causal inference rule in two steps.  First, we construct
a suitable and nonparametric representation of each sample $S_i$. Second, we
train a nonlinear binary classifier on such features to distinguish between
causal directions.  Building upon this framework, we derive theoretical
guarantees regarding consistency and learning rates, extend inference to the
multivariate case, propose approximations to scale learning to big data, and
obtain state-of-the-art performance with a simple
implementation.

Given the ubiquity of uncertainty in data, which may arise from noisy
measurements or the existence of unobserved common causes, we adopt the
probabilistic interpretation of causation from \citet{Pearl00}. 
Under this interpretation, the causal structure underlying a set of random
variables $X = (X_1, \ldots, X_d)$, with joint distribution $P$, is often
described in terms of a Directed Acyclic Graph (DAG), denoted by $G = (V,E)$.
In this graph, each vertex $V_i \in V$ is associated to the random variable
$X_i \in X$, and an edge $E_{ji} \in E$ from $V_j$ to $V_i$ denotes 
the causal relationship ``$X_i \leftarrow X_j$''. More specifically, these causal
relationships are defined by a \emph{structural equation model}: each
$X_i \leftarrow f_i(\text{Pa}(X_i), N_i)$, where $f_i$ is a function, $\text{Pa}(X_i)$
is the parental set of $V_i \in V$, and
$N_i$ is some independent noise variable.
Then, causal inference is the task of recovering $G$ from $S \sim
P^n$.  

\subsection{Prior Art}

We now briefly review the state-of-the-art on the inference of causal
structures $G$ from observational data $S\sim P^n$. For a more thorough
exposition, see, e.g., \citet{Mooij14}.

One of the main strategies to recover $G$ is through the exploration of
conditional dependencies, together with some other technical assumptions such
as the \emph{Markov} and \emph{faithfulness} relationships between $P$ and $G$
\citep{Pearl00}. This is the case of the PC algorithm
\citep{Spirtes00}, which allows the recovery of the Markov equivalence class of
$G$ without placing any restrictions on the structural equation model
specifying the random variables under study. 

Causal inference algorithms that exploit conditional
dependencies are unsuited for inference in the bivariate case.  Consequently,
a large body of work has been dedicated to the study of this scenario.  First,
the linear non-Gaussian causal model \citep{Shimizu06,Shimizu11} recovers the
true causal direction between two variables whenever their relationship is
linear and polluted with additive and non-Gaussian noise. This model was later extended into 
nonlinear additive noise models
\citep{Hoyer09,Zhang09,Stegle10,Kpotufe13,Peters14:ANM}, which prefer the causal
direction under which the alleged cause is independent from the additive
residuals of some nonlinear fit to the alleged effect.  Third, the information
geometric causal inference framework \citep{Daniusis12,Janzing14} assumes that
the cause random variable is independently generated from some invertible and
deterministic mapping to its effect; thus, it is unlikely to find dependencies
between the density of the former and the slope of the latter, under the
correct causal direction.

As it may be inferred from the previous exposition, there exists a large and
heterogeneous array of causal inference algorithms, each of them working under
a very specialized set of assumptions, which are sometimes difficult to test in
practice. Therefore, there exists the need for a more flexible causal inference
rule, capable of learning the relevant \emph{causal
footprints}, later used for inference, directly from data. Such a ``data driven'' approach would
allow to deal with complex data-generating processes, and would greatly reduce 
the need of explicitly crafting identifiability conditions a-priori.

A preliminary step in this direction distilled from the
competitions organized by \citet{Kaggle13,Codalab14}, which phrased causal
inference as a learning problem.  In these competitions, the participants were
provided with a large collection of \emph{cause-effect samples}
$\{(S_i,l_i)\}_{i=1}^n$, where $S_i = \{(x_{ij},y_{ij})\}_{j=1}^{n_i}$ is drawn
from the probability distribution of $X_i \times Y_i$, and $l_i$ is a binary
label indicating whether ``$X_i \to Y_i$'' or ``$Y_i \to X_i$''. Given these
data, most participants adopted the strategy of i) crafting a vector of
features from each $S_i$, and ii) training a binary classifier on top of the
constructed features and paired labels. Although these ``data-driven'' methods
achieved state-of-the-art performance \citep{Kaggle13}, the laborious task of
hand-crafting features renders their theoretical analysis impossible.

In more specific terms, the approach described above is a learning problem with
inputs being sample sets $S_i$, where each $S_i$ contains samples drawn from the
probability distribution $P_i(X_i,Y_i)$.  In a separate strand of research,
there has been several attempts to learn from probability distributions in a
principled manner \citep{Jebara04,Hein04,Cuturi05,Martins09,Muandet12}.
\citet{Szabo14b} presented the first theoretical analysis of distributional
learning based on kernel mean embedding \citep{Smola07Hilbert}, with focus on
kernel ridge regression. Similarly, \citet{Muandet12} studied the problem of
classifying distributions, but their approach is constrained to kernel
machines, and no guarantees regarding consistency or learning rates are
provided.

\subsection{Our Contribution}

Inspired by Guyon's competitions, we pose causal inference as the problem of
classifying probability measures on causally related pairs of random variables.
Our contribution to this framework is the use of \emph{kernel mean embeddings}
to nonparametrically {featurize} each cause-effect sample $S_i$.  The benefits
of this approach are three-fold. First, this avoids the need of
hand-engineering features from the samples $S_i$. Second, this enables a
clean theoretical analysis, including provable learning rates and consistency
results.  Third, the kernel hyperparameters (that is, the data representation)
can be jointly optimized with the classifier using cross-validation.
Furthermore, we show how to extend these ideas to infer causal relationships
between $d \geq 2$ variables, give theoretically sustained approximations to scale
learning to big data, and provide the source code of a simple implementation
that outperforms the state-of-the-art.

The rest of this article is organized as follows. Section~\ref{sec:embeddings}
reviews the concept of kernel mean embeddings, the tool that will facilitate
learning from distributions.  Section~\ref{sec:theory} shows the consistency
and learning rates of our kernel mean embedding classification approach to
cause-effect inference.  Section~\ref{sec:dags} extends the presented 
ideas to the multivariate causal inference case. Section~\ref{sec:exps}
presents a variety of experiments displaying the state-of-the-art performance
of a simple implementation of the proposed framework.
For convenience, Table~\ref{table:notation} summarizes our notations.
\begin{table}[h!]\footnotesize
  \begin{center}
    \resizebox{\columnwidth}{!}{  
    \begin{tabular}{|l|l|}
    \hline
    $\mathbb{E}[\xi], \mathbb{V}[\xi]$ & Expected value and variance of r.v. $\xi$\\
    \hline\hline
     $\Z$ & Domain of cause-effect pairs $Z = (X,Y)$\\
     \hline
     $\P$ & Set of cause-effect measures $P$ on $\Z$\\
     \hline
     $\L$ & Set of labels $l_i\in\{-1,1\}$\\
     \hline
     $\M$ & Mother distribution over $\P\times\L$\\
     \hline
     $\{(P_i,l_i)\}_{i=1}^n$ & Sample from $\M^n$\\
     \hline
     $S_i\!=\!\{Z_{ij}\}_{j=1}^{n_i}$ & Sample from $P_i^{n_i}$\\
     \hline
     $P_{S_i}$ & Empirical distribution of $S_i$\\
     \hline\hline
     $k$ & Kernel function from $\Z \times \Z$ to $\R$\\
     \hline
     $\H_k$ & RKHS induced by $k$\\
     \hline
     $\mu_k(P)$ & Kernel mean embedding of measure $P\in \P$\\
     \hline
     $\mu_k(P_{s_i})$ & Empirical mean embedding of $P_{s_i}$\\
     \hline
     $\mu_k(\P)$ & The set $\{ \mu_k(P) : P \in \P \}$\\
     \hline
     $\M_k$ & Measure over $\mu_k(\P) \times \L$ induced by $\M$\\
     \hline\hline
	   $\F_k$ & Class of functionals mapping $\H_k$ to $\R$\\
     \hline
     $R_n(\F_k)$ & Rademacher complexity of class $\F_k$\\
     \hline
     $\varphi$, $\Rp(f)$ & Cost and surrogate $\varphi$-risk of $\sig f$\\
     \hline
    \end{tabular}
    }
  \end{center}
  \vskip -0.25 cm
  \caption{Table of notations}
  \label{table:notation}
\end{table}

\section{Kernel Mean Embeddings of Probability Measures}
\label{sec:embeddings}

In order to later classify probability measures $P$ according to their causal
properties, we first need to featurize them into a suitable representation. To
this end, we will rely on the concept of \emph{kernel mean embeddings}
\cite{Berlinet04:RKHS,Smola07Hilbert}.

In particular, let $P$ be the probability distribution of some random variable
$Z$ taking values in the separable topological space $(\Z,\tau_z)$.  Then, the
\emph{kernel mean embedding} of $P$ associated with the continuous, bounded,
and positive-definite kernel function $k : \Z \times \Z \to \R$ is 
\begin{equation}
  \label{eq:meank}
  \mu_k(P) := \int_{\Z} k( z, \cdot) \, \d P(z),
\end{equation}
which is an element in $\H_k$, the Reproducing Kernel Hilbert Space (RKHS) associated with $k$
\cite{Schoelkopf02}.  Interestingly, the mapping $\mu_k$ is injective if $k$ is
a characteristic kernel \citep{Sriperumbudur10:Metrics}, that is,
$\|\mu_k(P)-\mu_k(Q)\|_{\H_k}=0 \Leftrightarrow P=Q$. Said differently, if
using a characteristic kernel, we do not lose any information when embedding
distributions. An example of characteristic kernel is the Gaussian kernel
\begin{equation}\label{eq:gauss}
  k(z,z') = \exp\left(-\gamma \|z-z'\|_2^2\right), \,\, \gamma > 0,
\end{equation}
which will be used throughout this paper.

In many practical situations, it is unrealistic to assume access to the true
distribution $P$, and consequently to the true embedding $\mu_k(P)$.  Instead,
we often have access to a sample $S = \{z_i\}_{i=1}^n \sim P^n$, which can be
used to construct the empirical distribution $P_S := \frac{1}{n} \sum_{z_i\in
S} \delta_{(z_i)}$, where $\delta_{(z)}$ is the Dirac distribution centered at
$z$.  Using $P_S$, we can approximate \eqref{eq:meank} by the \emph{empirical
kernel mean embedding}
\begin{equation}
  \label{eq:meank2} 
  \mu_k(P_S) := \frac{1}{n} \sum_{i=1}^n k( z_i, \cdot) \in \H_k.
\end{equation}
The following result is a slight modification of Theorem 27 from \cite{S06}.
It establishes the convergence of the empirical mean embedding $\mu_k(P_S)$ to
the embedding of its population counterpart $\mu_k(P)$, in RKHS norm:
\begin{theorem}
\label{thm:LeSong}
Assume that $\|f\|_{\infty}\leq 1$ for all $f\in \H_k$ with $\|f\|_{\H_k}\leq
1$.  Then with probability at least $1-\delta$ we have
\[
\|\mu_k(P)-\mu_k(P_S)\|_{\H_k}
\leq
2\sqrt{\frac{\E_{z\sim P}[k(z,z)]}{n}} + \sqrt{\frac{2\log\frac{1}{\delta}}{n}}.
\]
\end{theorem}
\begin{proof}See Section~\ref{proof:LeSong}.\end{proof}

\section{A Theory of Causal Inference as Distribution Classification}\label{sec:theory}
This section phrases the inference of cause-effect relationships from 
probability measures as the classification of empirical kernel mean embeddings,
and analyzes the learning rates and consistency of such approach.  Throughout
our exposition, the setup is as follows:
\begin{enumerate}
  \item We assume the existence of some \emph{Mother distribution} $\M$,
  defined on $\P \times \L$, where $\P$ is the set of all Borel probability
  measures on the space $\Z$ of two causally related random variables, and $\L
  = \{-1,+1\}$. 

  \item A set $\{(P_i, l_i)\}_{i=1}^n$ is sampled from $\M^n$. Each
  measure $P_i \in \P$ is the joint distribution of the causally related random
  variables $Z_i = (X_i, Y_i)$, and the label $l_i \in \L$ indicates whether
  ``$X_i \to Y_i$'' or ``$X_i \leftarrow Y_i$''.
  
  \item In practice, we do not have access to the measures $\{P_i\}_{i=1}^n$.
  Instead, we observe samples $S_i = \{(x_{ij}, y_{ij})\}_{j=1}^{n_i} \sim
  P_i^{n_i}$, for all $1 \leq i \leq n$. Using $S_i$, the data
  $\{(S_i,l_i)\}_{i=1}^n$ is provided to the learner. 

  \item We featurize every sample $S_i$ into the empirical kernel mean
  embedding $\mu_k(P_{S_i})$ associated with some kernel function $k$
  (Equation~\ref{eq:meank2}). If $k$ is a characteristic kernel, we incur no
  loss of information in this step.
\end{enumerate}

Under this setup, we will use the set $\{(\mu_k(P_{S_i}),l_i)\}_{i=1}^n$ to
train a binary classifier from $\H_k$ to $\L$, which will later be used to
unveil the causal directions of new, unseen probability measures drawn from
$\M$.  Note that this framework can be straightforwardly extended to also infer
the ``confounding ($X \leftarrow Z \rightarrow Y$)'' and ``independent ($X
\indep Y$)'' cases by adding two extra labels to~$\mathcal{L}$.

Given the two nested levels of sampling (being the first one from the Mother
distribution $\M$, and the second one from each of the drawn cause-effect measures
$P_i$), it is not trivial to conclude whether this learning procedure is consistent, or how
its learning rates depend on the sample sizes $n$ and $\{n_i\}_{i=1}^n$.  In
the following, we will study the generalization performance of empirical risk
minimization over this learning setup.  Specifically, we are interested in
upper bounding the {excess risk} between the empirical risk minimizer and
the best classifier from our hypothesis class, with respect to the Mother
distribution $\M$.

We divide our analysis in three parts.  First, \S\ref{sec:classic} reviews~the
abstract setting of statistical learning theory and surrogate risk
minimization.  Second, \S\ref{sec:distrib} adapts these standard results to the
case of empirical kernel mean embedding classification.  Third,
\S\ref{sec:random} considers theoretically sustained approximations to deal with big data.

\subsection{Margin-based Risk Bounds in Learning Theory}
\label{sec:classic}
Let $\Pio$ be some unknown probability measure defined on $\Z \times \L$, where
$\Z$ is referred to as the \emph{input space}, and $\L = \{-1,1\}$ is referred
to as the \emph{output space}\footnote{Refer to Section~\ref{sec:measurability} for considerations on measurability.}.
One of the main goals of statistical learning theory
\cite{Vap98} is to find a classifier $h\colon \Z\to\L$ that minimizes the
\emph{expected risk}
\begin{equation*}
  R(h) = \E_{(z,l)\sim \Pio}\bigl[\ell\bigl(h(z),l\bigr)\bigr]
\end{equation*}
for a suitable \emph{loss function} $\ell\colon\L \times \L\to\R^+$,~which 
  penalizes departures between predictions $h(z)$ and true labels~$l$.
  For classification, one common choice of loss function is the \emph{0-1 loss}
  $\ell_{01}(l,l')=|l-l'|$, for which the expected risk measures the
  probability of misclassification.  Since~$\Pio$ is unknown in natural
  situations, one usually resorts to the minimization of the \emph{empirical
  risk} $\frac{1}{n}\sum_{i=1}^n\ell\bigl(h(z_i),l_i\bigr)$ over some fixed
  hypothesis class $\H$, for \emph{the training set} $\{(z_i,l_i)\}_{i=1}^n
  \sim \Pio^n$.  It is well known that this procedure~is consistent under mild~assumptions \cite{BBL05}.

Unfortunately, the 0-1 loss function is not convex, which leads to empirical
risk minimization being generally intractable.  Instead, we will focus on the
minimization of surrogate risk functions \cite{BJM06}.  In particular, we will
consider the set of classifiers of the form $\Hyp = \{{\sig f}\colon
{f\in\F}\}$ where $\F$ is some fixed set of real-valued functions ${f\colon
\Z\to\R}$.  Introduce a nonnegative \emph{cost function} $\varphi\colon
\R\to\R^+$ which is surrogate to the 0-1 loss, that is, $\varphi(\epsilon) \geq
\mathbbm{1}_{\epsilon>0}$.  For any $f\in \F$ we define its expected and
empirical $\varphi$-risks respectively as
\begin{equation}\label{eq:phirisk1}
  \Rp(f) = \E_{(z,l)\sim \Pio}\bigl[\varphi\bigl(-f(z) l\bigr)\bigr],
\end{equation}
\begin{equation}\label{eq:phirisk2}
  \Rpn(f) = \frac{1}{n}\sum_{i=1}^n \varphi\bigl(-f(z_i) l_i\bigr).
\end{equation}
Many natural choices of $\varphi$ lead to tractable empirical risk
minimization.  Common examples of cost functions include the \emph{hinge loss}
$\varphi(\epsilon)=\max(0,1+\epsilon)$ used in SVM, the \emph{exponential loss}
$\varphi(\epsilon)=\exp(\epsilon)$ used in Adaboost, and the \emph{logistic
loss} $\varphi(\epsilon)=\log_2\bigl(1+e^\epsilon)$ used in logistic
regression.

The misclassification error of $\sig f$ is always upper bounded by $\Rp(f)$.
The relationship between functions minimizing $\Rp(f)$ and functions minimizing
${R(\sig f)}$ has been intensively studied in the literature \citep[Chapter
3]{Steinwart08}.  Given the high uncertainty associated with causal inferences,
we argue that one is interested in predicting soft probabilities rather than
hard labels, a fact that makes the study of margin-based classifiers well
suited for our problem.

We now focus on the estimation of $f^*\in\F$, the function minimizing
\eqref{eq:phirisk1}.  However, since the distribution $\Pio$ is unknown, we can
only hope to estimate $\hat{f}_n\in\F$, the function minimizing
\eqref{eq:phirisk2}.  Therefore, we are interested in high-probability upper
bounds on the
\emph{excess $\varphi$-risk}
\begin{equation}
\label{eq:excess}
\Ex_{\F}(\fn) =
\Rp(\fn) - \Rp(\f),
\end{equation}
w.r.t. the random training sample $\{(z_i,l_i)\}_{i=1}^n \sim \Pio^n$.  
The excess risk \eqref{eq:excess} can be upper bounded in the following way:
\begin{align}
\Ex_{\F}(\fn)
\notag
&\leq
\Rp(\fn) - \Rpn(\fn)
+
\Rpn(\f) - \Rp(\f)\\
\label{eq:SLT-Bad}
&\leq
2\sup_{f\in\F}|\Rp(f) - \Rpn(f)|.
\end{align}
While this upper bound leads to tight results for worst case analysis, it is
well known \cite{BBM05,BBL05,K11} that tighter bounds can be achieved under
additional assumptions on $\Pio$.  However, we leave these analyses for future
research.

The following result --- in spirit of \citet{KP99,BM01} --- can be found in
\citet[Theorem 4.1]{BBL05}.
\begin{theorem}
\label{thm:classic}
Consider a class $\F$ of functions mapping $\Z$ to~$\R$.  Let $\varphi\colon
\R\to\R^+$ be a $L_\varphi$-Lipschitz function such that $\varphi(\epsilon)
\geq \mathbbm{1}_{\epsilon>0}$.  Let $B$ be a uniform upper bound on
$\varphi\bigl(-f(\epsilon) l\bigr)$.  Let $\{(z_i,l_i)\}_{i=1}^n \sim \Pio$ and
$\{\sigma_i\}_{i=1}^n$ be i.i.d. Rademacher random signs.  Then, with
prob. at least $1-\delta$,
\begin{multline*}
\sup_{f\in\F}|\Rp(f) - \Rpn(f)|\\
\leq
2 L_\varphi \E\left[\sup_{f\in \F}\frac{1}{n}\left|\sum_{i=1}^n \sigma_i f(z_i)\right|\right] + B\sqrt{\frac{\log(1/\delta)}{2n}},
\end{multline*}
where the expectation is taken w.r.t. $\{\sigma_i,z_i\}_{i=1}^n$.
\end{theorem}
The expectation in the bound of Thm.~\ref{thm:classic} is known~as \emph{the
Rademacher complexity} of $\F$, will be denoted by $R_n(\F)$, and has a typical
order of $O(n^{-1/2})$ \cite{K11}.

\subsection{From Classic to Distributional Learning Theory}\label{sec:distrib}

Note that we can not directly apply the empirical risk minimization bounds
discussed in the previous section to our learning setup.  This is because
instead of learning a classifier on the i.i.d.  sample
$\{\mu_k(P_i),l_i\}_{i=1}^n$, we have to learn over the set
$\{\mu_k(P_{S_i}),l_i\}_{i=1}^n$, where $S_i \sim P_i^{n_i}$.  Said
differently, our input feature vectors $\mu_k(P_{S_i})$ are ``noisy'': they
exhibit an additional source of variation as any two different random samples
$S_i,S_i'\sim P_i^{n_i}$ do.  In the following, we study how to incorporate
these nested sampling effects into an argument similar to
Theorem~\ref{thm:classic}.

We will now frame our problem within the abstract learning setting considered
in the previous section.  Recall that our learning setup initially considers
some Mother distribution $\M$ over $\P \times \L$.  Let $\mu_k(\P) = \{\mu_k(P)
: P \in \P \} \subseteq \H_k$, $\L = \{-1,+1\}$, and $\M_k$ be a
measure (guaranteed to exist by
Lemma~\ref{lemma:measurability}, 
Section~\ref{sec:measurability_distrib}) on $\mu_k(\P)
\times \L$ induced by $\M$.  Specifically, we will consider $\muP\subseteq
\H_k$ and $\L$ to be the input and output spaces of our learning problem,
respectively.  Let $\bigl\{\bigl(\mu_k(P_i),l_i\bigr)\bigr\}_{i=1}^n\sim
\M_k^n$ be our training set.  We will now work with the set of classifiers
$\{\sig f\colon f\in \F_k\}$ for some fixed class $\F_k$ of functionals mapping
from the RKHS $\H_k$ to $\R$.

As pointed out in the description of our learning setup, we do not have 
access to the distributions $\{P_i\}_{i=1}^n$ but to samples $S_i \sim
P_i^{n_i}$, for all $1 \leq i \leq n$.  Because of this reason, we define the
\emph{sample-based empirical $\varphi$-risk}
\[
\Rpnt(f) = \frac{1}{n}\sum_{i=1}^n \varphi\bigl(-l_i f\bigl(\mu_k(P_{S_i})\bigr)\bigr),
\]
which is the approximation to the empirical $\varphi$-risk $\Rpn(f)$ that
results from substituting the embeddings $\mu_k(P_i)$ with their empirical
counterparts $\mu_k(P_{S_i})$.

Our goal is again to find the function $\f\in\F_k$ minimizing expected
$\varphi$-risk $\Rp(f)$.  Since $\M_k$ is unknown to us, and we have no access
to the embeddings $\{\mu_k(P_i)\}_{i=1}^n$, we will instead use the minimizer
of $\Rpnt(f)$ in $\F_k$:
\begin{equation}
\label{eq:sample-based-erm}
\fnt \in \arg\min_{f\in \F_k} \Rpnt(f).
\end{equation}
To sum up, the excess risk \eqref{eq:excess} can now be reformulated as
\begin{equation}
\label{eq:excess2}
\Rp(\fnt) - \Rp(\f).
\end{equation}
Note that the estimation of $\f$ drinks from two nested sources of error, which
are i) having only $n$ training samples from the distribution $\M_k$, and ii)
having only $n_i$ samples from each measure $P_i$.  Using a similar technique
to \eqref{eq:SLT-Bad}, we can upper bound the excess risk as
\begin{align}
\label{eq:excess-term1}
\Rp(\fnt) - \Rp(\f)
&\leq
\sup_{f\in\F_k}|\Rp(f) - \Rpn(f)|\\
\label{eq:excess-term2}
&+
\sup_{f\in\F_k}|\Rpn(f) - \Rpnt(f)|.
\end{align}
The term \eqref{eq:excess-term1} is upper bounded by Theorem~\ref{thm:classic}.
On the other hand, to deal with \eqref{eq:excess-term2}, we will need to upper
bound the deviations
$\bigl|f\bigl(\mu_k(P_i)\bigr)-f\bigl(\mu_k(P_{S_i})\bigr)\bigr|$ in terms of
the distances $\|\mu_k(P_i) - \mu_k(P_{S_i})\|_{\H_k}$, which are in turn upper
bounded using Theorem \ref{thm:LeSong}.  To this end, we will have to assume
that the class $\F_k$ consists of functionals with uniformly bounded Lipschitz
constants,  such as the set of linear
functionals with uniformly bounded operator norm.

We now present the main result of this section, which provides a
high-probability bound on the excess risk \eqref{eq:excess2}. Importantly, this
excess risk will translate into the expected causal inference accuracy of our
distribution classifier.

\begin{theorem}
\label{thm:risk-bound}
Consider the RKHS $\H_k$ associated with some bounded, continuous
kernel function $k$, such that $\sup_{z \in \Z} k(z,z)\leq 1$.
Consider a class $\F_k$ of functionals mapping $\H_k$ to~$\R$ with Lipschitz
constants uniformly bounded by $L_{\F}$.  Let $\varphi\colon \R\to\R^+$ be a
$L_{\varphi}$-Lipschitz function such that $\phi(z) \geq \mathbbm{1}_{z>0}$.
Let $\varphi\bigl(-f(h) l\bigr) \leq B$ for every $f\in\F_k$, $h \in \H_k$, and
$l\in\L$.  Then, with probability not less than $1-\delta$ (over all sources of
randomness) 
\begin{align*}
&\Rp(\fnt) - \Rp(\f)
\leq
4 L_\varphi R_n(\F_k) + 2B\sqrt{\frac{\log(2/\delta)}{2n}}\\
&+
\frac{4L_\varphi L_{\F}}{n}
\sum_{i=1}^n
\left(
\sqrt{\frac{\E_{z\sim P_i}[k(z,z)]}{n_i}} + \sqrt{\frac{\log(2n/\delta)}{2n_i}}
\right).
\end{align*}
\end{theorem}
\begin{proof}See Section \ref{sect:ProofRiskBound}.\end{proof}

As mentioned in Section~\ref{sec:classic}, the typical order of $R_n(\F_k)$ is
$O(n^{-1/2})$. 
For a particular examples of classes of functionals with small Rademacher complexity we refer to \citet{Maurer06}.
In such cases, the upper bound in Theorem~\ref{thm:risk-bound}
converges to zero (meaning that our procedure is consistent) as both $n$ and
$n_i$ tend to infinity, in such a way that\footnote{
We conjecture that this constraint is an artifact from our proof.} $\log n/n_i = o(1)$.  The rate of
convergence w.r.t. $n$ can be improved up to $O(n^{-1})$ if placing additional
assumptions on $\M$ \cite{BBM05}. On the contrary, the rate w.r.t. $n_i$ cannot
be improved in general. Namely, the convergence rate $O(n^{-1/2})$ presented in the upper bound of Theorem~\ref{thm:LeSong} is
tight, as shown in the following novel result.
\begin{theorem}
\label{thm:lower-bound}
Under the assumptions of Theorem~\ref{thm:LeSong} denote
\[
\sigma^2_{\H_k} = \sup_{\|f\|_{\H_k}\leq 1} \mathbb{V}_{z\sim P}[f(z)].
\]
Then there exist universal constants $c,C$ such that for every integer $n\geq
1/\sigma^2_{\H_k}$, and with probability at least $c$
\[
\|\mu_k(P) - \mu_k(P_S)\|_{\H_k}
\geq
C\frac{\sigma_{\H_k}}{\sqrt{n}}.
\]  
\end{theorem}
\begin{proof}See Section \ref{sect:ProofLowerBound}.\end{proof}

Finally, it is instructive to relate the notion of ``identifiability'' often
considered in the causal inference community \citep{Pearl00} to the properties
of the Mother distribution.  Saying that the model is \emph{identifiable} means
that the label $l$ of $P\in\P$ is assigned deterministically by $\M$.  In this
case, learning rates can become as fast as $O(n^{-1})$.  On the other hand, as
$\M(l|P)$ becomes nondeterministic, the problem becomes unidentifiable and
learning rates slow down (for example, in the extreme case of cause-effect pairs related
by linear functions polluted with additive Gaussian noise, $\M(l=+1|P) =
\M(l=-1|P)$ almost surely). The investigation of these phenomena is left
for future research.

\subsection{Low Dimensional Embeddings for Large Data}\label{sec:random}

For some kernel functions, the embeddings $\mu_k(P_S)\in\H_k$
are infinite dimensional.  Because of this reason, one must resort to the use
of dual optimization problems, and in particular, kernel matrices.  The
construction of these matrices requires at least $O(n^2)$ computational and
memory requirements, prohibitive for large $n$.  In this section, we
show that the infinite-dimensional embeddings $\mu_k(P_S)\in\H_k$ can be
approximated with easy to compute, low-dimensional representations
\citep{Rahimi07,Rahimi09}.  This will allow us to replace the
infinite-dimensional minimization problem~\eqref{eq:sample-based-erm} with a
low-dimensional one.

Assume that $\Z = \R^d$, and that the kernel function $k$ is
real-valued, and shift-invariant. Then, we can exploit Bochner's theorem
\citep{Rudin62} to show that, for any $z,z'\in \Z$:
\begin{align}
\label{eq:bochner-simplified}
k(z,z')
\!=
2C_k\!\E_{w,b}\!\left[\cos(\langle w, z\rangle\!+b)\cos(\langle w, z'\rangle\!+b)\right],
\end{align}
where $w \sim \frac{1}{C_k}p_k$, $b \sim \U[0,2\pi]$, $p_k\colon\Z\to\R$ is the
positive and integrable Fourier transform of $k$, and $C_k=\int_{\Z}p_k(w)dw$.
For example, the squared-exponential kernel \eqref{eq:gauss} is
shift-invariant, and its evaluations can be approximated by
\eqref{eq:bochner-simplified}, if setting $p_k (w) = \N(w|0,2\gamma I)$, and
$C_k=1$.  

We now show that for any probability measure $Q$ on $\Z$~and $z\in\Z$, the
function $k(z,\cdot)\in\H_k\subseteq L_2(Q)$ can be approximated by a linear
combination of randomly chosen elements from the Hilbert space $L_2(Q)$.
Namely, consider the functions parametrised by $w,z\in\Z$ and $b\in[0,2\pi]$:
\begin{equation}
\label{eq:cosine-feature}
g_{w,b}^z (\cdot) = 2C_k\cos(\langle w,z\rangle + b)\cos(\langle w,\cdot\rangle + b),
\end{equation}
which belong to $L_2(Q)$, since they are bounded.  If we sample
$\{(w_j,b_j)\}_{j=1}^m$ i.i.d., as discussed above, the average 
\[
\hat{g}_m^z(\cdot) = \frac{1}{m}\sum_{i=1}^m g_{w_i,b_i}^z(\cdot)
\]
can be viewed as an $L_2(Q)$-valued random variable.  Moreover,
\eqref{eq:bochner-simplified} shows that $\E_{w,b}[\hat{g}_m^z(\cdot)] =
k(z,\cdot)$.  This enables us to invoke concentration inequalities for Hilbert
spaces \cite{LT91}, to show the following result, which is in spirit to
\citet[Lemma 1]{Rahimi09}.
\begin{lemma}
\label{lemma:approx}
Let $\Z=\R^d$.  For any shift-invariant kernel~$k$, s.t.
$\sup_{z\in\Z}k(z,z)\leq 1$, any fixed $S=\{z_i\}_{i=1}^n\subset \Z$, any
probability distribution $Q$ on $\Z$, and any $\delta > 0$, we have
\[
\Biggl\|\mu_k(P_S) - \frac{1}{n}\sum_{i=1}^n\hat{g}_m^{z_i}(\cdot)\Biggr\|_{L_2(Q)}
\!\!\!\!\!\leq 
\frac{2C_k}{\sqrt{m}}\left(1 + \sqrt{{2\log(n/\delta)}}\right)
\]
with probability larger than $1-\delta$ over $\{(w_i,b_i)\}_{i=1}^m$.
\end{lemma}
\begin{proof}See Section \ref{proof:ApproxLemma}.\end{proof}

Once sampled, the parameters $\{(w_i,b_i)\}_{i=1}^m$ allow us to approximate
the empirical kernel mean embeddings $\{\mu_k(P_{S_i})\}_{i=1}^n$ using
elements from $\text{span}({\{\cos(\langle w_i, \cdot\rangle +
b_i)\}_{i=1}^m})$, which is a finite-dimensional subspace of $L_2(Q)$.  
Therefore, we propose to
use $\{(\mu_{k,m}(P_{S_i}),l_i)\}_{i=1}^n$ as the training sample for our final
empirical risk minimization problem, where
\begin{equation}\label{eq:meank3}
  \mu_{k,m}(P_S) = \frac{2C_k}{|S|} \sum_{z\in S} \bigl( \cos(\langle w_j,
  z\rangle + b_j)\bigr)_{j=1}^m \!\!\in\!
  \R^m.
\end{equation}

These feature vectors can be computed in $O(m)$ time and stored in $O(1)$
memory; importantly, they can be used off-the-shelf in conjunction with any
learning algorithm.

For the precise excess risk bounds that take into account the use of these
low-dimensional approximations, please refer to Theorem~\ref{thm:new-theorem} from
Section~\ref{sec:new-theorem}.

\section{Extensions to Multivariate Causal Inference}\label{sec:dags}

It is possible to extend our framework to infer causal relatonships
between $d \geq 2$ variables $X=(X_1,\ldots,X_d)$. To this end, and as
introduced in Section~\ref{sec:intro}, assume the existence of a causal
directed acyclic graph $G$ which underlies the dependencies present in the
probability distribution $P(X)$. Therefore, our task is to recover $G$
from $S \sim P^n$.

Na\"ively, one could extend the framework presented in Section~\ref{sec:theory}
from the binary classification of $2$-dimensional distributions to the
multiclass classification of $d$-dimensional distributions. However, the number
of possible DAGs (and therefore, the number of labels in our multiclass
classification problem) grows super-exponentially in $d$.

An alternative approach is to consider the probabilities of the three labels
``$X_i \to X_j$'', ``$X_i \leftarrow X_j$``, and ``$X_i \indep X_j$`` for each
pair of variables $\{X_i,X_j\}\subseteq X$, when embedded along with every
possible \emph{context} $X_k \subseteq X \setminus \{X_i,X_j\}$. The intuition
here is the same as in the PC algorithm of \citet{Spirtes00}: in order to
decide the (absence of a) causal relationship between $X_i$ and $X_j$, one must
analyze the confounding effects of every $X_k \subseteq X \setminus
\{X_i,X_j\}$. 

\section{Numerical Simulations}\label{sec:exps}

We conduct an array of experiments to test the effectiveness of a simple implementation 
of the presented causal learning framework.  Given the use of random embeddings 
\eqref{eq:meank3} in our classifier, we term our
method the \emph{Randomized Causation Coefficient} (RCC).  Throughout our
simulations, we featurize each sample $S =
\{(x_{i},y_{i})\}_{i=1}^{n}$ as
\begin{align}\label{eq:ourfeatz}
  \nu(S) = (\mu_{k,m}(P_{S_x}), \mu_{k,m}(P_{S_y}), \mu_{k,m}(P_{S_{xy}})),
\end{align}
where the three elements forming \eqref{eq:ourfeatz} stand for the
low-dimensional representations \eqref{eq:meank3} of the empirical kernel mean
embeddings of $\{x_i\}_{i=1}^n$, $\{y_i\}_{i=1}^n$, and
$\{(x_i,y_i)\}_{i=1}^n$, respectively. The representation \eqref{eq:ourfeatz}
is motivated by the typical conjecture in causal inference about the existence
of asymmetries between the marginal and conditional distributions of
causally-related pairs of random variables \citep{Scholkopf12:Causal}. Each of
these three embeddings has random features sampled to approximate the sum of
three Gaussian kernels \eqref{eq:gauss} with hyper-parameters $0.1\gamma$,
$\gamma$, and $10\gamma$, where $\gamma$ is found using the median heuristic.
In practice, we set $m = 1000$, and observe no significant improvements when
using larger amounts of random features. To classify the embeddings
\eqref{eq:ourfeatz} in each of the experiments, we use the random
forest\footnote{Although random forests do not comply with Lipschitzness
assumptions from Section~\ref{sec:theory}, they showed the best empirical results.
Compliant alternatives such as SVMs exhibited a typical drop in classification
accuracy of $5\%$.} implementation from Python's \texttt{sklearn-0.16-git}. The
number of trees is chosen from 
$\{100,250,500,1000,5000\}$ via cross-validation.

Our experiments can be replicated using the source code 
at 

\resizebox{\linewidth}{!}{{\url{https://github.com/lopezpaz/causation_learning_theory}}.}

\subsection{Classification of T\"ubingen Cause-Effect Pairs}\label{sec:tuebingen}

The \emph{T\"ubingen cause-effect pairs} is a collection of heterogeneous,
hand-collected, real-world cause-effect samples \cite{Tuebingen14}.  Given the
small size of this dataset, we resort to the synthesis of an artificial Mother distribution
to sample our training data from.  To this end, assume that
sampling a synthetic cause-effect sample set $\hat{S}_i := \{(\hat{x}_{ij},
\hat{y}_{ij})\}_{j=1}^n \sim {\P}_\theta$ equals the following simple
generative process:
\begin{enumerate}
  \item A \emph{cause} vector $(\hat{x}_{ij})_{j=1}^{n}$ is sampled from a
  mixture of Gaussians with $c$ components. The mixture weights are
  sampled from $\U(0,1)$, and normalized to sum to one. The mixture means and
  standard deviations are sampled from $\N(0,\sigma_1)$, and $\N(0,\sigma_2)$,
  respectively, accepting only positive standard deviations.
  The cause vector is standardized to zero mean and unit variance.
  \item A \emph{noise} vector $(\hat{\epsilon}_{ij})_{j=1}^{n}$ is sampled from a
  centered Gaussian, with variance sampled from $\U(0,\sigma_3)$.
  \item A \emph{mapping mechanism} $\hat{f}_i$ is conceived as a spline fitted
  using an uniform grid of $d_f$ elements from
  $\min((\hat{x}_{ij})_{j=1}^n)$ to $\max((\hat{x}_{ij})_{j=1}^n)$ as inputs, and 
  $d_f$ normally distributed outputs.
  \item An \emph{effect} vector is built as $(\hat{y}_{ij} :=
  \hat{f}_i(\hat{x}_{ij})+\hat{\epsilon}_{ij})_{j=1}^n$, and
  standardized to zero mean and unit variance.
  \item Return the cause-effect sample $\hat{S}_i := \{(\hat{x}_{ij},
  \hat{y}_{ij})\}_{j=1}^n$.
\end{enumerate}
To choose a $\theta = (c,\sigma_1, \sigma_2, \sigma_3, d_f)$
that best resembles the unlabeled test data, we minimize the distance between the
embeddings of $N$ synthetic pairs and the Tuebingen samples 
\begin{align}\label{eq:Motherobj}
  \argmin_\theta \sum_{i}^{} \min_{1 \leq j \leq N} \| \nu(S_i)-\nu(\hat{S}_j) \|^2_2,
\end{align}
over $c,d_f \in \{1, \ldots, 10\}$, and $\sigma_1, \sigma_2$, and $\sigma_3 \in \{0,
0.5, 1, \ldots, 5\}$, where the $\hat{S}_j \sim \P_\theta$, the $S_i$ are the
T\"ubingen cause-effect pairs, and $\nu$ is as in \eqref{eq:ourfeatz}.
This strategy can be thought of as transductive learning, since we assume to
know the test inputs prior to the training of our inference rule. 
We set $n = 1000$, and $N = 10,000$.

Using the generative process outlined above, we construct the synthetic
training data 
\begin{align*}
  \{&\{\nu(\{(\hat{x}_{ij}, \hat{y}_{ij})\}_{j=1}^{n}), +1)\}_{i=1}^{N},\\
  &\{\nu(\{(\hat{y}_{ij}, \hat{x}_{ij})\}_{j=1}^{n}), -1)\}_{i=1}^{N}\},
\end{align*}
where $\{(\hat{x}_{ij}, \hat{y}_{ij})\}_{j=1}^{n} \sim \P_\theta$, and train
our classifier on it.

\begin{figure}[t]
  \begin{center}
    \includegraphics[width=\linewidth]{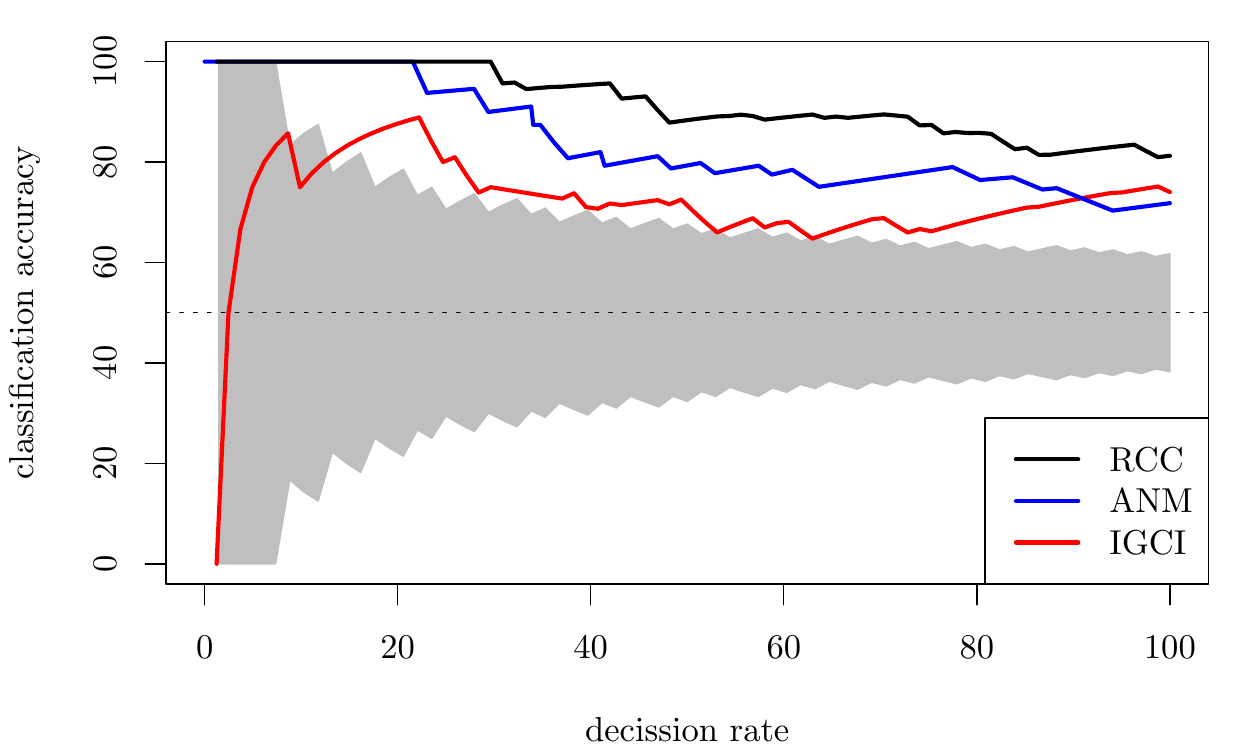} 
  \end{center}
  \caption{Accuracy of RCC, IGCI and ANM on the T\"ubingen cause-effect pairs,
  as a function of decision rate. The grey area depicts accuracies not
  statistically significant.}
  \label{fig:tuebingen}
\end{figure}

Figure \ref{fig:tuebingen} plots the classification accuracy of RCC, IGCI
\citep{Daniusis12}, and ANM \citep{Mooij14} versus the fraction of decissions 
that the algorithms are forced to make out of the 82 scalar T\"uebingen
cause-effect pairs. To compare these results to other lower-performing methods,
refer to \citet{Janzing12}. RCC surpasses the state-of-the-art 
with a classification accuracy of $81.61\%$ when inferring the
causal directions on all pairs. The confidence of RCC is computed using the
classifier's output class probabilities. SVMs obtain a test accuracy of
$77.2\%$ in this same task.

\subsection{Inferring the Arrow of Time}

We test the effectiveness of our method to infer the arrow of time from causal
time series. More specifically, we assume access to a set of time series $
\{x_{ij}\}_{j=1}^{n_i}$, and our task is to infer, for each series, whether $X_i
\to X_{i+1}$ or $X_i \leftarrow X_{i+1}$.

We compare our framework to the state-of-the-art of \citet{Peters09}, using the
same electroencephalography signals \citep{EEG} as in their original
experiment.  On the one hand, \citet{Peters09} construct two Auto-Regressive
Moving-Average (ARMA) models for each causal time series and time direction,
and prefers the solution under which the model residuals are independent from
the inferred cause. To this end, the method uses two parameters for which no
estimation procedure is provided.  On the other hand, our approach makes no
assumptions whatsoever about the parametric model underlying the series, at the
expense of requiring a disjoint set of $N=10,000$ causal time series for
training. Our method matches the best performance of \citet{Peters09}, with an
accuracy of $82.66\%$.
 
\subsection{ChaLearn's Challenge Data}

The cause-effect challenges organized by \citet{Codalab14} provided $N =
16,199$ training causal samples $S_i$, each drawn from the distribution of $X_i
\times Y_i$, and labeled either ``$X_i \to Y_i$'', ``$X_i \leftarrow Y_i$'',
``$X_i \leftarrow Z_i \to Y_i$'', or ``$X_i \indep Y_i$''. The task of the
competition was to develop a \emph{causation coefficient} which would predict
large positive values to causal samples following ``$X_i \to Y_i$'', large
negative values to samples following ``$X_i \leftarrow Y_i$'', and zero
otherwise.  Using these data, our obtained a test
\textit{bidirectional area under the curve score} \citep{Codalab14} of $0.74$
in one minute and a half, ranking third in the overall leaderboard.  The winner of the competition obtained a score of
$0.82$ in thirty minutes, but resorted to several dozens of hand-crafted
features.

Partitioning these same data in different ways, we learned two related but
different binary classification tasks. First, we trained our classifier to
\emph{detect latent confounding}, and obtained a test classification accuracy
of $80\%$ on the task of distinguishing ``$X \to Y$ or $X \leftarrow X$'' from
``$X \leftarrow Z \to Y$''.  Second, we trained our classifier to \emph{measure
dependence}, and obtained a test classification accuracy of $88\%$ on the task
of distinguishing between ``$X \indep Y$'' and ``else''. We consider these
results to be a promising direction to learn flexible hypothesis tests and
dependence measures \emph{directly from data}.

\subsection{Reconstruction of Causal DAGs} \label{sec:dags-experiment}

We apply the strategy described in Section~\ref{sec:dags} to reconstruct the
causal DAGs of two multivariate datasets: \emph{autoMPG} and \emph{abalone}
\citep{UCI}. Once again, we resort to synthetic training data, generated in a
similar procedure to the one used in Section~\ref{sec:tuebingen}. Refer to
Section~\ref{sec:dagtrain} for details.

Regarding \emph{autoMPG}, in
Figure~\ref{fig:auto}, 1) the release date of the vehicle (AGE)
causes the miles per gallon consumption (MPG), acceleration capabilities (ACC)
and horse-power (HP), 2) the weight of the vehicle (WEI) causes the horse-power
and MPG, and that 3) other characteristics such as the engine displacement
(DIS) and number of cylinders (CYL) cause the MPG.  For \emph{abalone}, in
Figure~\ref{fig:abalone}, 1) the age of the snail causes all
the other variables, 2) the overall weight of the snail (WEI) is caused by the
partial weights of its meat (WEA), viscera (WEB), and shell (WEC), and 3) the
height of the snail (HEI) is responsible for other phisicaly attributes such as
its diameter (DIA) and length (LEN).

The target variable for 
each dataset is shaded in gray. Interstingly, our inference
reveals that the \emph{autoMPG} dataset is a \emph{causal} prediction task (the
features \emph{cause} the target), and that the \emph{abalone} dataset is an
\emph{anticausal} prediction task (the target \emph{causes} the features). This
distinction has implications when learning from these data \citep{Scholkopf12}.

\begin{figure}[h!]
\begin{center}
\begin{tikzpicture}[node distance=1cm, auto,]
 \node[punkt,fill=gray!30!white] (MPG) {MPG};
 \node[punkt, below=of MPG] (AGE) {AGE};
 \node[punkt, left=of AGE] (ACC) {ACC};
 \node[punkt, left=of ACC] (WEI) {WEI};
 \node[punkt, above=of WEI] (HP)  {HP};
 \node[punkt, right=of AGE] (CYL) {CYL};
 \node[punkt, right=of MPG] (DIS) {DIS};
 \draw[pil] (WEI) -- (MPG);
 \draw[pil] (CYL) -- (MPG);
 \draw[pil] (DIS) -- (MPG);
 \draw[pil] (HP) -- (MPG);
 \draw[pil] (AGE) -- (MPG);
 \draw[pil] (AGE) -- (ACC);
 \draw[pil] (AGE) -- (HP);
 \draw[pil] (WEI) -- (HP);
\end{tikzpicture}
\end{center}
\vspace{-0.2 cm}
\caption{Causal DAG recovered from data \emph{autoMPG}.}
\label{fig:auto}
\end{figure}
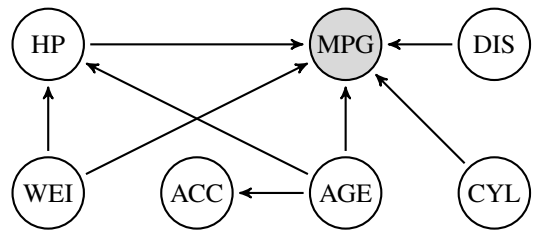
\vspace{-0.1cm}
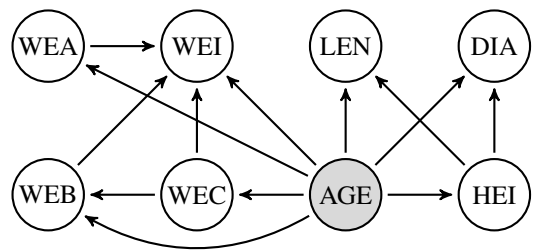
\begin{figure}[h!]
\begin{center}
\begin{tikzpicture}[node distance=1cm, auto,]
 \node[punkt,fill=gray!30!white] (AGE) {AGE};
 \node[punkt, left=of AGE] (WEC) {WEC};
 \node[punkt, left=of WEC] (WEB) {WEB};
 \node[punkt, above=of WEB] (WEA) {WEA};
 \node[punkt, above=of AGE] (LEN) {LEN};
 \node[punkt, right=of LEN] (DIA) {DIA};
 \node[punkt, right=of AGE] (HEI) {HEI};
 \node[punkt, right=of WEA] (WEI) {WEI};
 \draw[pil] (AGE) -- (LEN);
 \draw[pil] (AGE) -- (HEI);
 \draw[pil] (AGE) -- (WEI);
 \draw[pil] (AGE) -- (WEA);
 \draw[pil] (AGE) -- (WEC);
 \draw[pil] (WEC) -- (WEB);
 \draw[pil] (WEC) -- (WEI);
 \draw[pil] (WEA) -- (WEI);
 \draw[pil] (WEB) -- (WEI);
 \draw[pil] (HEI) -- (LEN);
 \draw[pil] (HEI) -- (DIA);
 \draw[pil] (AGE) -- (DIA);
 \draw[pil] (AGE) to[bend left] (WEB);
\end{tikzpicture}
\end{center}
\vspace{-0.2 cm}
\caption{Causal DAG recovered from data \emph{abalone}.}
\label{fig:abalone}
\end{figure}
\section{Future Work}\label{sec:discussion}

Three research directions are in progress. First, to improve learning rates by
using common assumptions from causal inference.  Second, to further investigate
methods to reconstruct multivariate DAGs. Third, to develop mechanisms to
interpret the causal footprints learned by our classifiers. 

\newpage
\clearpage
\bibliography{jmlr_rcc}
\bibliographystyle{icml2015}

\newpage
\clearpage
\onecolumn
\appendix

\section{Topological and Measurability Considerations}\label{sec:measurability}
Let $(\Z, \tau_\Z)$ and $(\L, \tau_\L)$ be two separable topological spaces,
where $\Z$ is the \emph{input space} and $\L := \{-1,1\}$ is the \emph{output
space}.  Let $\B(\tau)$ be the Borel \hbox{$\sigma$-algebra} induced by the
topology $\tau$.  Let $\Pio$ be an unknown probability measure on $(\Z \times
\L,\B(\tau_\Z)\otimes\B(\tau_\L))$. 

Consider also the classifiers $f \in \F_k$ and loss function $\ell$ to be measurable.

\subsection{Measurability Conditions to Learn from Distributions}\label{sec:measurability_distrib}

The first step towards the deployment of our learning setup is to guarantee the
existence of a measure on the space $\muP \times \L$, where $\muP = \{\mu_k(P)
: P \in \P \} \subseteq \H_k$ is the set of kernel mean embeddings associated
with the measures in $\P$.  The following lemma provides such guarantee.  This
allows the analysis within the rest of this Section on ${\muP \times \L}$.

\begin{lemma}\label{lemma:measurability}
  Let $(\Z, \tau_\Z)$ and $(\L, \tau_\L)$ be two separable topological spaces.
  Let $\P$ be the set of all Borel probability measures on $(\Z,\B(\tau_\Z))$.
  Let $\muP = \{ \mu_k(P) : P \in \P \} \subseteq \H_k$, where $\mu_k$ is the
  kernel mean embedding \eqref{eq:meank} associated to some bounded continuous
  kernel function $k : \Z \times \Z \to \R$.  Then, there exists a measure on
  $\muP \times \L$.
  \begin{proof}
    The following is a similar result to \citet[Proof 3]{Szabo14b}.
    
    Start by endowing $\P$ with the weak topology $\tau_\P$, such that the map
    \begin{equation}\label{eq:lmap}
      L(P) = \int_\Z f(z)\d P(z),
    \end{equation}
    is continuous for all $f \in C_b(\Z)$. This makes $(\P,\B(\tau_\P))$ a
    measurable space.
    
    First, we show that $\mu_k : (\P,\B(\tau_\P)) \to (\H_k,\B(\tau_\H))$ is Borel
    measurable. Note that $\H_k$ is separable due to the separability of $(\Z,
    \tau_\Z)$ and the continuity of $k$ \citep[Lemma 4.33]{Steinwart08}. The
    separability of $\H_k$ implies $\mu_k$ is Borel measurable iff it is weakly
    measurable \citep[Thm.  IV.22]{Reed72}.  Note that the boundedness and the
    continuity of $k$ imply $\H_k \subseteq C_b(\Z)$ \citep[Lemma
    4.28]{Steinwart08}. Therefore, \eqref{eq:lmap} remains continuous for all $f\in
    \H_k$, which implies the Borel measurability of $\mu_k$.
    
    Second, $\mu_k : (\P, \B(\tau_\P)) \to (\G, \B(\tau_\G))$ is Borel measurable,
    since the $\B(\tau_\G) = \{ A \cap \G : A \in \B(\H_k)\} \subseteq
    \B(\tau_\H)$, where $\B(\tau_\G)$ is the $\sigma$-algebra induced by the
    topology of $\G \in \B(\H_k)$ \citep{Szabo14b}.
    
    Third, we show that $g : (\P \times \L, \B(\tau_\P) \otimes \B(\tau_\L)) \to
    (\G \times \L, \B(\tau_\G) \otimes \B(\tau_\L))$ is measurable. For that, it
    suffices to decompose $g(x,y) = (g_1(x,y),g_2(x,y))$ and show that $g_1$ and
    $g_2$ are measurable \citep{Szabo14b}.
  \end{proof}
\end{lemma}

\section{Proofs}\label{sec:proofs}

\subsection{Theorem \ref{thm:LeSong}}
\label{proof:LeSong}
Note that the original statement of Theorem 27 in \citet{S06} assumed
$f\in[0,1]$ while we let elements of the ball in RKHS to take negative values
as well which can be achieved by minor changes of the proof.  For completeness
we provide the modified proof here.  Using the well known dual relation between
the norm in RKHS and sup-norm of empirical process which can be found in
Theorem 28 of \citet{S06} we can write:
\begin{equation}
\label{eq:RKHS-duality}
\|\mu_k(P)-\mu_k(P_S)\|_{\H_k}
=
\sup_{\|f\|_{\H_k}\leq 1}\left(
\E_{z\sim P}[f(z)]
-
\frac{1}{n}\sum_{i=1}^n f(z_i)
\right).
\end{equation}
Now we proceed in the usual way.  First we note that the sup-norm of empirical
process appearing on the r.h.s. can be viewed as a real-valued function of
i.i.d.\:random variables $z_1,\dots,z_n$.  We will denote it as
$F(z_1,\dots,z_n)$.  The straightforward computations show that the function
$F$ satisfies the \emph{bounded difference} condition (Theorem 14 of
\citet{S06}). Indeed, let us fix all the values $z_1,\dots,z_n$ except for the
$z_j$ which we will set to $z_j'$.  Using identity $|a-b| =
(a-b)\mathbbm{1}_{a>b} + (b-a)\mathbbm{1}_{a\leq b}$ and noting that if $\sup_x
f(x) = f(x^*)$ then $\sup_x f(x) - \sup_x g(x)$ is upper bounded by $f(x^*) -
g(x^*)$ we get
\begin{align*}
&|F(z_1,\dots,z_j',\dots,z_n) - F(z_1,\dots,z_j,\dots,z_n)|
\\
&\leq
\frac{1}{n}\bigl(f(z_j) - f(z_j')\bigr)\mathbbm{1}_{F(z_1,\dots,z_j',\dots,z_n) > F(z_1,\dots,z_j,\dots,z_n)}
+
\frac{1}{n}\bigl(f(z_j') - f(z_j)\bigr)\mathbbm{1}_{F(z_1,\dots,z_j',\dots,z_n) \leq F(z_1,\dots,z_j,\dots,z_n)}.
\end{align*}
Now noting that $|f(z) - f(z')|\in[0,2]$ we conclude with
\begin{align*}
&
|F(z_1,\dots,z_j',\dots,z_n) - F(z_1,\dots,z_j,\dots,z_n)|\\
&\leq 
\frac{2}{n}\mathbbm{1}_{F(z_1,\dots,z_j',\dots,z_n) > F(z_1,\dots,z_j,\dots,z_n)}
+
\frac{2}{n}\mathbbm{1}_{F(z_1,\dots,z_j',\dots,z_n) \leq F(z_1,\dots,z_j,\dots,z_n)} = \frac{2}{n}.
\end{align*}
Using McDiarmid's inequality (Theorem 14 of \cite{S06}) with $c_i=2/n$ we
obtain that with probability not less than $1-\delta$ the following holds:
\[
\sup_{\|f\|_{\H_k}\leq 1}\left(
\E_{z\sim P}[f(z)]
-
\frac{1}{n}\sum_{i=1}^n f(z_i)
\right)
\leq
\E\left[\sup_{\|f\|_{\H_k}\leq 1}\left(
\E_{z\sim P}[f(z)]
-
\frac{1}{n}\sum_{i=1}^n f(z_i)
\right)\right]
+
\sqrt{\frac{2\log(1/\delta)}{n}}.
\]
Finally, we proceed with the symmetrization step (Theorem 2.1 of \cite{K11})
which upper bounds the expected value of the sup-norm of empirical process with
twice the Rademacher complexity of the class $\{f\in\H_k\colon \|f\|_{\H_k}\leq
1\}$ and with upper bound on this Rademacher complexity which can be found in
Lemma 22 and related remarks of \citet{BM01}.

We also note that the original statement of Theorem 27 in \citet{S06} contains
extra factor of 2 under logarithm compared to our modified result.  This is
explained by the fact that while we upper bounded the Rademacher complexity
directly, \citet{S06} instead upper bounds it in terms of the
empirical (or \emph{conditional}) Rademacher complexity which results in
another application of McDiarmid's inequality together with union bound.

\subsection{Theorem \ref{thm:risk-bound}}\label{sect:ProofRiskBound}
We will proceed as follows:
\begin{align}
\notag
\Rp(\fnt) - \Rp(\f)
&=
\Rp(\fnt) - \Rpnt(\fnt)\\
\notag
&+
\Rpnt(\fnt) -\Rpnt(\f)\\
\notag
&+
\Rpnt(\f) - \Rp(\f)\\
\notag
&\leq
2\sup_{f\in \F_k}|\Rp(f) - \Rpnt(f)|\\
\notag
&=
2\sup_{f\in \F_k}|\Rp(f) - \Rpn(f) + \Rpn(f) - \Rpnt(f)|\\
&\leq
\label{eq:excess-bound1}
2\sup_{f\in \F_k}|\Rp(f) - \Rpn(f)|
+
2\sup_{f\in \F_k}|\Rpn(f) - \Rpnt(f)|.
\end{align}
We will now upper bound two terms in \eqref{eq:excess-bound1} separately.

We start with noticing that Theorem \ref{thm:classic} can be used in order to
upper bound the first term.  All we need is to match the quantities appearing
in our problem to the classical setting of learning theory, discussed in
Section \ref{sec:classic}.  Indeed, let $\mu(\P)$ play the role of input space
$\Z$.  Thus the input objects are kernel mean embeddings of elements of $\P$.
According to Lemma~\ref{lemma:measurability}, there is a distribution defined
over $\mu(\P)\times\L$, which will play the role of unknown distribution
$\Pio$.  Finally, i.i.d. pairs
$\bigl\{\bigl(\mu_k(P_i),l_i\bigr)\bigr\}_{i=1}^n$ form the training sample.
Thus, using Theorem \ref{thm:classic} we get that with probability not less
than $1-\delta/2$ (w.r.t. the random training sample
$\bigl\{\bigl(\mu_k(P_i),l_i\bigr)\bigr\}_{i=1}^n$) the following holds true:
\begin{equation}
\label{eq:excess-bound2}
\sup_{f\in\F_k}|\Rp(f) - \Rpn(f)|\\
\leq
2 L_\varphi \E\left[\sup_{f\in \F_k}\frac{1}{n}\left|\sum_{i=1}^n \sigma_i f(z_i)\right|\right] + B\sqrt{\frac{\log(2/\delta)}{2n}}.
\end{equation}
To deal with the second term in \eqref{eq:excess-bound1} we note that
\begin{align*}
\notag
\sup_{f\in \F_k}|\Rpn(f) - \Rpnt(f)| 
&=
\sup_{f\in \F_k}\left|
\frac{1}{n}
\sum_{i=1}^n\Bigl[\varphi\bigl(-l_if\bigl(\mu_k(P_i)\bigr)\bigr) - \varphi\bigl(-l_i f\bigl(\mu_k(P_{S_i})\bigr)\bigr)\Bigr]
\right| \\
\notag
&\leq
\sup_{f\in \F_k}
\frac{1}{n}
\sum_{i=1}^n
\left|
\varphi\bigl(-l_if\bigl(\mu_k(P_i)\bigr)\bigr) - \varphi\bigl(-l_i f\bigl(\mu_k(P_{S_i})\bigr)\bigr)
\right| \\
\notag
&\leq
L_\varphi\sup_{f\in \F_k}
\frac{1}{n}
\sum_{i=1}^n
\left|
f\bigl(\mu_k(P_i)\bigr) - f\bigl(\mu_k(P_{S_i})\bigr)
\right|,
\end{align*}
where we have used the Lipschitzness of the cost function $\varphi$.
Using the Lipschitzness of the functionals $f\in\F_k$ we obtain:
\begin{align}
\label{eq:sec3-proof-1}
\sup_{f\in \F_k}|\Rpn(f) - \Rpnt(f)| \leq
L_\varphi\sup_{f\in \F_k}
\frac{L_f}{n}
\sum_{i=1}^n
\|
\mu_k(P_i) - \mu_k(P_{S_i})
\|_{\H_k}.
\end{align}
Also note that the usual reasoning shows that if $h\in\H_k$ and
$\|h\|_{\H_k}\leq 1$ then: 
\[
|h(z)|=|\langle h, k(z,\cdot) \rangle_{\H_k}| \leq \|h\|_{\H_k} \|k(z,\cdot)\|_{\H_k} = \|h\|_{\H_k} \sqrt{k(z,z)}\leq \sqrt{k(z,z)}
\]
and hence $\|h\|_{\infty} = \sup_{z\in\Z}|h(z)| \leq 1$ because our kernel is
bounded.  This allows us to use Theorem \ref{thm:LeSong} to control every term
in \eqref{eq:sec3-proof-1} and combine the resulting upper bounds in a union
bound\footnote{
Note that the union bound results in the extra $\log n$ factor in our bound. 
We believe that this factor can be avoided using a refined proof technique, based on the application of McDiarmid's inequality.
This question is left for a future work.
} over $i=1,\dots,n$ to show that for any fixed $P_1,\dots,P_n$ with
probability not less than $1-\delta/2$ (w.r.t. the random samples
$\{S_i\}_{i=1}^n$) the following is true:
\begin{equation}
\label{eq:excess-bound3}
L_\varphi\sup_{f\in \F}
\frac{L_f}{n}
\sum_{i=1}^n
\|
\mu_k(P_i) - \mu_k(P_{S_i})
\|_{\H_k}
\leq
L_\varphi\sup_{f\in \F}
\frac{L_f}{n}
\sum_{i=1}^n
\left(
2\sqrt{\frac{\E_{z\sim P}[k(z,z)]}{n_i}} + \sqrt{\frac{2\log\frac{2n}{\delta}}{n_i}}
\right).
\end{equation}
The quantity $2n/\delta$ appears under the logarithm since for every $i$ we
have used Theorem \ref{thm:LeSong} with $\delta' = \delta / (2n)$.  Combining
\eqref{eq:excess-bound2} and \eqref{eq:excess-bound3} in a union bound together
with \eqref{eq:excess-bound1} we finally get that with probability not less
than $1-\delta$ the following is true:
\[
\Rp(\fnt) - \Rp(\f)
\leq
4 L_\varphi R_n(\F) + 2B\sqrt{\frac{\log(2/\delta)}{2n}}
+
\frac{4L_\varphi L_{\F}}{n}
\sum_{i=1}^n
\left(
\sqrt{\frac{\E_{z\sim P}[k(z,z)]}{n_i}} + \sqrt{\frac{\log\frac{2n}{\delta}}{2n_i}}
\right),
\]
where we have defined $L_{\F} = \sup_{f\in\F}L_f$.

\subsection{Theorem \ref{thm:lower-bound}}
\label{sect:ProofLowerBound}
Our proof is a simple combination of the duality equation
\eqref{eq:RKHS-duality} combined with the following lower bound on the supremum
of empirical process presented in Theorem 2.3 of \citet{BM06}:
\begin{theorem}
\label{thm:BM06}
Let $F$ be a class of real-valued functions defined on a set $\Z$ such that
$\sup_{f\in F}\|f\|_{\infty}\leq 1$.  Let $z_1,\dots,z_n,z\in\Z$ be
i.i.d.\:according to some probability measure $P$ on $\Z$.  Set $\sigma^2_F =
\sup_{f\in F} \mathbb{V}[f(z)].$ Then there are universal constants $c,c',$ and
$C$ for which the following holds:
\[
\E\left[\sup_{f\in F}\left|
\E[f(z)]
-
\frac{1}{n}\sum_{i=1}^n f(z_i)
\right|\right]
\geq
c\frac{\sigma_F}{\sqrt{n}}.
\]
Furthermore, for every integer $n\geq 1/\sigma^2_F$, with probability at least $c'$,
\[
\sup_{f\in F}\left|
\E[f(z)]
-
\frac{1}{n}\sum_{i=1}^n f(z_i)
\right|
\geq
C\E\left[\sup_{f\in F}\left|
\E[f(z)]
-
\frac{1}{n}\sum_{i=1}^n f(z_i)
\right|\right].
\]
\end{theorem}
We note that constants $c,c',$ and $C$ appearing in the last result do not
depend on $n,\sigma^2_F$ or any other quantities appearing in the statement.
This can be verified by the inspection of the proof presented in \cite{BM06}.

\subsection{Lemma \ref{lemma:approx}}
\label{proof:ApproxLemma}
\begin{proof}
Bochner's theorem \citep{Rudin62} states that for any shift-invariant symmetric
p.d. kernel $k$ defined on $\Z\times \Z$ where $\Z=R^d$ and any $z,z'\in \Z$
the following holds:
\begin{equation}
\label{eq:bochner-first}
k(z,z') = \int_{\Z}p_k(w) e^{i\langle w, z-z'\rangle} dw,
\end{equation}
where $p_k$ is a positive and integrable Fourier transform of the kernel $k$.
It is immediate to check that Fourier transform of such kernels $k$ is always
an even function, meaning $p_k(-w)=p_k(w)$.  Indeed, since $k(z-z')=k(z'-z)$
for all $z,z'\in\Z$ (due to symmetry of the kernel) we have:
\[
p_k(w) := \int_{\Z}k(\delta) e^{i\langle w, \delta\rangle}d\delta
=
\int_{\Z}k(\delta) \cos(\langle w, \delta\rangle)d\delta
=
\int_{\Z}k(\delta) \cos(-\langle w, \delta\rangle)d\delta
=
p_k(-w)
\]
which holds for any $w\in \mathbb{R}^d$.
Thus for any $z,z'\in \mathbb{R}^d$ we can write:
\begin{align}
\notag
k(z,z') &= \int_{\mathbb{R}^d}p_k(w) \bigl(\cos(\langle w, z-z'\rangle)+i\cdot \sin(\langle w, z-z'\rangle)\bigr) dw\\
\notag
&=\int_{\mathbb{R}^d}p_k(w) \bigl(\cos(\langle w, z-z'\rangle) dw +i\cdot \int_{\mathbb{R}^d}p_k(w)\sin(\langle w, z-z'\rangle)\bigr) dw\\
\notag
&=\int_{\mathbb{R}^d}p_k(w) \cos(\langle w, z-z'\rangle) dw\\
\notag
&=2\int_{\mathbb{R}^d}\int_{0}^{2\pi}\frac{1}{2\pi}p_k(w) \cos(\langle w, z\rangle+b)\cos(\langle w, z'\rangle+b)\, db\, dw.
\end{align}
Denote $C_k=\int_{\mathbb{R}^d}p(w)dw<\infty$.  Next we will use identity
$\cos(a-b)=\frac{1}{\pi}\int_{0}^{2\pi}\cos(a+x)\cos(b+x)dx$ and introduce
random variables $b$ and $w$ distributed according to $\mathcal{U}[0,2\pi]$ and
$\frac{1}{C_k}p_k(w)$ respectively.  Then we can rewrite
\begin{align}
\label{eq:bochner-simple}
k(z,z')=2C_k\E_{b,w}\left[ \cos(\langle w, z\rangle+b)\cos(\langle w, z'\rangle+b)\right].
\end{align}

Now let $Q$ be any probability distribution defined on $\Z$. 
Then for any $z,w\in\Z$ and $b\in[0,2\pi]$ the function 
\[
g_{w,b}^z(\cdot) := 2 C_k \cos(\langle w, z\rangle+b)\cos(\langle w, \cdot\rangle+b)
\]
belongs to the $L_2(Q)$ space. 
Namely, $L_2(Q)$ norm of such a function is finite.
Moreover, it is bounded by $2C_k$:
\begin{align}
\notag
\|g_{w,b}^z(\cdot)\|^2_{L_2(Q)}&=\int_{\Z}\Bigl(2C_k\cos(\langle w, z\rangle+b)\cos(\langle w, t\rangle+b)\Bigr)^2 dQ(t)\\
\label{eq:fbounded}
&\leq
4C_k^2\int_{\Z}dQ(t) = 4C_k^2.
\end{align}
Note that for any fixed $x\in\Z$ and any random parameters $w\in\Z$ and
$b\in[0,2\pi]$ the element $g_{w,b}^z(\cdot)$ is a \emph{random variable}
taking values in the $L_2(Q)$ space (which is Hilbert).  Such Banach-space
valued random variables are well studied objects \cite{LT91} and a number of
concentration results for them are known by now.  We will use the following
version of Hoeffding inequality which can be found in Lemma 4 of
\citet{Rahimi09}:
\begin{lemma}
\label{lemma:rahimi_hoeffding}
Let $v_1,\dots,v_m$ be i.i.d.\:random variables taking values in a ball of
radius $M$ centred around origin in a Hilbert space $H$. 
Then, for any $\delta > 0$, the following holds: 
\[
\left\|\frac{1}{m}\sum_{i=1}^m v_i - \E\left[\frac{1}{m}\sum_{i=1}^m v_i \right]\right\|_{H} \leq \frac{M}{m}\left(1+\sqrt{2\log(1/\delta)}\right).
\]
with probability higher than $1-\delta$ over the random sample $v_1,\dots,v_m$.
\end{lemma}
Note that Bochner's formula \eqref{eq:bochner-first} and particularly its
simplified form \eqref{eq:bochner-simple} indicates that if $w$ is distributed
according to normalized Fourier transform $\frac{1}{C_k}p_k$ and
$b\sim\mathcal{U}([0,2\pi])$ then $\E_{w,b}[g_{w,b}^z(\cdot)] = k(z,\cdot)$.
Moreover, we can show that any element $h$ of RKHS $\H_k$ also belongs to the
$L_2(Q)$ space:
\begin{align}
\notag
\|h(\cdot)\|^2_{L_2(Q)}
&=\int_{\Z}\bigl(h(t)\bigr)^2 dQ(t)\\
\notag
&=\int_{\Z}\langle k(t,\cdot),h(\cdot)\rangle_{\H_k}^2 dQ(t)\\
\label{eq:hk-in-l2}
&\leq\int_{\Z}k(t,t)\|h\|_{\H_k}^2 dQ(t) \leq \|h\|_{\H_k}^2<\infty,
\end{align}
where we have used the reproducing property of $k$ in RKHS $\H_k$,
Cauchy-Schwartz inequality, and the fact that the kernel $k$ is bounded.  Thus
we conclude that the function $k(z,\cdot)$ is also an element of $L_2(Q)$
space.

This shows that if we have a sample of i.i.d.\:pairs $\{(w_i,b_i)\}_{i=1}^m$
then $\E\left[\frac{1}{m}\sum_{i=1}^m g^z_{w_i,b_i}(\cdot)\right]= k(z,\cdot)$
where $\{g^z_{w_i,b_i}(\cdot)\}_{i=1}^m$ are i.i.d.\:elements of Hilbert space
$L_2(Q)$.  We conclude the proof using concentration inequality for Hilbert
spaces of Lemma \ref{lemma:rahimi_hoeffding} and a union bound over the
elements $z\in S$, since
\begin{align*}
\left\|\mu_k(P_S) - \frac{1}{n}\sum_{i=1}^n\hat{g}_m^{z_i}(\cdot)\right\|_{L_2(Q)}
&=
\left\|\frac{1}{n}\sum_{i=1}^nk(z_i,\cdot) - \frac{1}{n}\sum_{i=1}^n\hat{g}_m^{z_i}(\cdot)\right\|_{L_2(Q)}\\
&\leq
\frac{1}{n}\sum_{i=1}^n\left\|k(z_i,\cdot) - \hat{g}_m^{z_i}(\cdot)\right\|_{L_2(Q)}\\
&=
\frac{1}{n}\sum_{i=1}^n\left\|k(z_i,\cdot) - \frac{1}{m}\sum_{i=j}^m g^{z_i}_{w_j,b_j}(\cdot)\right\|_{L_2(Q)},
\end{align*}
where we have used the triangle inequality.
\end{proof}

\subsection{Excess Risk Bound for Low-Dimensional Representations}\label{sec:new-theorem}
Let us first recall some important notations introduced in Section \ref{sec:random}.
For any $w,z\in\Z$ and $b\in[0,2\pi]$ we define the following functions
\begin{equation}
\label{eq:new-proof-cos}
g_{w,b}^z (\cdot) = 2C_k\cos(\langle w,z\rangle + b)\cos(\langle w,\cdot\rangle + b) \in L_2(Q),
\end{equation}
where $C_k = \int_{\Z}p_k(z) dz$ for $p_k\colon \Z\to\R$ being the Fourier transform of $k$.
We sample $m$ pairs $\{(w_i,b_i)\}_{i=1}^m$ i.i.d. from $\left(\frac{1}{C_k}p_k\right)\times \U[0,2\pi]$ and define the average function
\[
\hat{g}_m^z(\cdot) = \frac{1}{m}\sum_{i=1}^m g_{w_i,b_i}^z(\cdot) \in L_2(Q).
\]

Since cosine functions \eqref{eq:new-proof-cos} do not necessarily belong to the RKHS $\H_k$ 
and we are going to use their linear combinations as a training points,
our classifiers should now act on the whole $L_2(Q)$ space.
To this end, we redefine the set of classifiers introduced in the Section \ref{sec:distrib} to 
be $\{\sig f\colon f\in\F_Q\}$ where now $\F_Q$ is the set of functionals mapping $L_2(Q)$ to $\R$.

Recall that our goal is to find $f^*$ such that
\begin{equation}
\label{eq:new-proof-1}
f^* \in \arg\min_{f\in \F_Q} R_{\varphi}(f):= \arg\min_{f\in \F_Q} \E_{(P,l)\sim \M}\left[\varphi\Bigl(-f\bigl(\mu_k(P)\bigr)l\Bigr)\right].
\end{equation}
As was pointed out in Section \ref{proof:ApproxLemma} if the kernel $k$ is bounded $\sup_{z\in\Z}k(z,z)\leq 1$ then $\H_k\subseteq L_2(Q)$.
In particular, for any $P\in\P$ it holds that $\mu_k(P)\in L_2(Q)$ and thus \eqref{eq:new-proof-1} is well defined.

Instead of solving \eqref{eq:new-proof-1} directly, we will again use the version of empirical risk minimization (ERM).
However, this time we won't use empirical mean embeddings $\{\mu_k(P_{S_i})\}_{i=1}^n$ 
since, as was already discussed, those lead to the expensive computations involving the kernel matrix.
Instead, we will pose the ERM problem in terms of the low-dimensional approximations based on cosines.
Namely, we propose to use the following estimator $\fnt^m$:
\begin{equation*}
\fnt^m \in \arg\min_{f\in \F_Q} \Rpnt^m(f):= \arg\min_{f\in \F_Q} \frac{1}{n}\sum_{i=1}^n\varphi\left(-f\left(\frac{1}{n_i}\sum_{z\in S_i}\hat{g}_m^{z}(\cdot)\right)l_i\right).
\end{equation*}
The following result puts together Theorem \ref{thm:risk-bound} and Lemma \ref{lemma:approx} to provide an excess risk bound for $\fnt^m$ which accounts for all sources of the errors introduced in the learning pipeline:
\begin{theorem}
\label{thm:new-theorem}
Let $\Z=\R^d$
and $Q$ be any probability distribution on $\Z$.
Consider the RKHS $\H_k$ associated with some bounded, continuous,
shift-invariant kernel function $k$, such that $\sup_{z \in \Z} k(z,z)\leq 1$.
Consider a class $\F_Q$ of functionals mapping $L_2(Q)$ to~$\R$ with Lipschitz
constants uniformly bounded by $L_{Q}$.  Let $\varphi\colon \R\to\R^+$ be a
$L_{\varphi}$-Lipschitz function such that $\phi(z) \geq \mathbbm{1}_{z>0}$.
Let $\varphi\bigl(-f(h) l\bigr) \leq B$ for every $f\in\F_Q$, $h \in L_2(Q)$, and
$l\in\L$.  
Then for any $\delta > 0$ the following holds:
\begin{align*}
\Rp(\fnt^m) - \Rp(\f)
&\leq
4 L_\varphi R_n(\F_Q) + 2B\sqrt{\frac{\log(3/\delta)}{2n}}\\
&+
\frac{4L_\varphi L_{Q}}{n}
\sum_{i=1}^n
\left(
\sqrt{\frac{\E_{z\sim P_i}[k(z,z)]}{n_i}} + \sqrt{\frac{\log\frac{3n}{\delta}}{2n_i}}
\right)\\
&+
2\frac{L_{\varphi}L_Q}{n}\sum_{i=1}^n
\frac{2C_k}{\sqrt{m}}\left(1 + \sqrt{{2\log(3n\cdot n_i/\delta)}}\right)
\end{align*}
with probability not less than $1-\delta$ over all sources of
randomness, which are $\{(P_i,l_i)\}_{i=1}^n$, $\{S_i\}_{i=1}^n$, $\{(w_i,b_i)\}_{i=1}^m$.
\end{theorem}
\begin{proof}
We will proceed similarly to \eqref{eq:excess-bound1}:
\begin{align}
\notag
\Rp(\fnt^m) - \Rp(\f)
&=
\Rp(\fnt^m) - \Rpnt^m(\fnt^m)\\
\notag
&+
\Rpnt^m(\fnt^m) -\Rpnt^m(\f)\\
\notag
&+
\Rpnt^m(\f) - \Rp(\f)\\
\notag
&\leq
2\sup_{f\in \F_Q}|\Rp(f) - \Rpnt^m(f)|\\
\notag
&=
2\sup_{f\in \F_Q}|\Rp(f) - \Rpn(f) + \Rpn(f) - \Rpnt(f) + \Rpnt(f) - \Rpnt^m(f)|\\
&\leq
\label{eq:new-proof-ex}
2\sup_{f\in \F_Q}|\Rp(f) - \Rpn(f)|
+
2\sup_{f\in \F_Q}|\Rpn(f) - \Rpnt(f)|
+
2\sup_{f\in \F_Q}|\Rpnt(f) - \Rpnt^m(f)|.
\end{align}
First two terms of \eqref{eq:new-proof-ex} were upper bounded in Section \ref{sect:ProofRiskBound}.
Note that the upper bound of the second term (proved in Theorem \ref{thm:risk-bound}) was based on the assumption that functionals in $F_Q$ are Lipschitz on $\H_k$ w.r.t.\:the $\H_k$ metric.
But as we already noted, for bounded kernels we have $\H_k\subseteq L_2(Q)$ which implies $\|h\|_{L_2(Q)}\leq \|h\|_{\H_k}$ for any $h\in\H_k$ (see \eqref{eq:hk-in-l2}). Thus $|f(h) - f(h')| \leq L_f \|h - h'\|_{L_2(Q)}\leq L_f \|h - h'\|_{\H_k}$ for any $h,h'\in \H_k$.
It means that the assumptions of Theorem \ref{thm:risk-bound} hold true and we can safely apply it to upper bound the first two terms of \eqref{eq:new-proof-ex}.

We are now going to upper bound the third one using Lemma \ref{lemma:approx}:
\begin{align*}
\sup_{f\in \F_Q}|\Rpnt(f) - \Rpnt^m(f)|&=
\sup_{f\in \F_Q}\left| \frac{1}{n}\sum_{i=1}^n \varphi\left(-f\bigl(\mu_k(P_{S_i})\bigr)l_i\right) - 
\frac{1}{n}\sum_{i=1}^n\varphi\left(-f\left(\frac{1}{n_i}\sum_{z\in S_i}\hat{g}_m^{z}(\cdot)\right)l_i\right)\right|\\
&\leq
\frac{1}{n}\sum_{i=1}^n\sup_{f\in \F_Q}\left| \varphi\left(-f\bigl(\mu_k(P_{S_i})\bigr)l_i\right) - 
\varphi\left(-f\left(\frac{1}{n_i}\sum_{z\in S_i}\hat{g}_m^{z}(\cdot)\right)l_i\right)\right|\\
&\leq
\frac{L_{\varphi}}{n}\sum_{i=1}^n\sup_{f\in \F_Q}\left| f\bigl(\mu_k(P_{S_i})\bigr) - 
f\left(\frac{1}{n_i}\sum_{z\in S_i}\hat{g}_m^{z}(\cdot)\right)\right|\\
&\leq
\frac{L_{\varphi}}{n}\sum_{i=1}^n\sup_{f\in \F_Q}L_f
\left\| \mu_k(P_{S_i}) - \frac{1}{n_i}\sum_{z\in S_i}\hat{g}_m^{z}(\cdot)\right\|_{L_2(Q)}.
\end{align*}
We can now use Lemma \ref{lemma:approx} combined in union bound over $i=1,\dots,n$ with $\delta' = \delta/n$.
This will give us that
\[
\sup_{f\in \F_Q}|\Rpnt(f) - \Rpnt^m(f)|
\leq
\frac{L_{\varphi}L_Q}{n}\sum_{i=1}^n
\frac{2C_k}{\sqrt{m}}\left(1 + \sqrt{{2\log(n\cdot n_i/\delta)}}\right).
\]
with probability not less than $1-\delta$ over $\{(w_i,b_i)\}_{i=1}^m$.
\end{proof}

\clearpage
\section{Training and Test Protocols for Section~\ref{sec:dags-experiment}}\label{sec:dagtrain}

The synthesis of the training data for the experiments described in
Section~\ref{sec:dags-experiment} follows a very similar procedure to the one from
Section~\ref{sec:tuebingen}. The main difference here is that, when trying to
infer the cause-effect relationship between two variables $X_i$ and $X_j$ belonging to a larger set of variables $X = (X_1, \ldots, X_d)$, we
will have to account for the effects of possible confounders $X_k \subseteq X \setminus \{X_i,X_j\}$.
For the sake of simplicity, we will only consider one-dimensional confounding
effects, that is, scalar $X_k$.

\subsection{Training Phase}

To generate cause-effect pairs exhibiting every possible scalar confounding
effect, we will generate data from the eight possible directed acyclic graphs
depicted in Figure~\ref{fig:8dags}.

\begin{figure}[h!]
  \begin{center}
  \begin{tikzpicture}[node distance=0.8cm, auto,]
   \node[punkt, minimum height=2em             ] (X1) {};
   \node[punkt, minimum height=2em, right=of X1] (Y1) {};
   \node[punkt, minimum height=2em, right=of Y1] (Z1) {};
   
   \node[punkt, minimum height=2em, below=of X1] (X2) {};
   \node[punkt, minimum height=2em, right=of X2] (Y2) {};
   \node[punkt, minimum height=2em, right=of Y2] (Z2) {};
   \draw[pil] (X2) -- (Y2);
   
   \node[punkt, minimum height=2em, below=of X2] (X3) {};
   \node[punkt, minimum height=2em, right=of X3] (Y3) {};
   \node[punkt, minimum height=2em, right=of Y3] (Z3) {};
   \draw[pil] (X3) -- (Y3);
   \draw[pil] (Y3) -- (Z3);
   
   \node[punkt, minimum height=2em, right=of Z1] (X4) {};
   \node[punkt, minimum height=2em, right=of X4] (Y4) {};
   \node[punkt, minimum height=2em, right=of Y4] (Z4) {};
   \draw[pil] (X4) -- (Y4);
   \draw[pil] (Z4) -- (Y4);
   
   \node[punkt, minimum height=2em, below=of X4] (X5) {};
   \node[punkt, minimum height=2em, right=of X5] (Y5) {};
   \node[punkt, minimum height=2em, right=of Y5] (Z5) {};
   \draw[pil] (Y5) -- (X5);
   \draw[pil] (Y5) -- (Z5);
   
   \node[punkt, minimum height=2em, below=of X5] (X6) {};
   \node[punkt, minimum height=2em, right=of X6] (Y6) {};
   \node[punkt, minimum height=2em, right=of Y6] (Z6) {};
   \draw[pil] (X6) -- (Y6);
   \draw[pil] (Y6) -- (Z6);
   \draw[pil] (X6) to[bend left] (Z6);
   
   \node[punkt, minimum height=2em, right=of Z4] (X7) {};
   \node[punkt, minimum height=2em, right=of X7] (Y7) {};
   \node[punkt, minimum height=2em, right=of Y7] (Z7) {};
   \draw[pil] (X7) -- (Y7);
   \draw[pil] (Z7) -- (Y7);
   \draw[pil] (X7) to[bend left] (Z7);
   
   \node[punkt, minimum height=2em, below=of X7] (X8) {};
   \node[punkt, minimum height=2em, right=of X8] (Y8) {};
   \node[punkt, minimum height=2em, right=of Y8] (Z8) {};
   \draw[pil] (Y8) -- (X8);
   \draw[pil] (Y8) -- (Z8);
   \draw[pil] (X8) to[bend left] (Z8);
  \end{tikzpicture}
  \end{center}
  \caption{The eight possible directed acyclic graphs on three variables.} 
  \label{fig:8dags}
\end{figure}
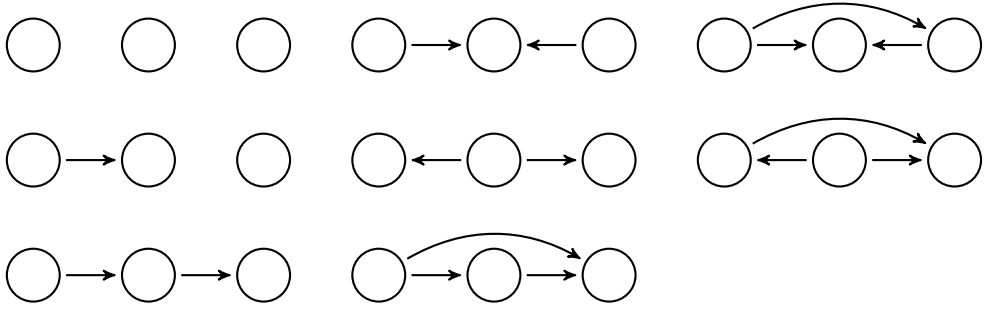

In particular, we will sample $N$ different causal DAGs $G_1, \ldots, G_N$,
where the $G_i$ describes the causal structure underlying $(X_i, Y_i, Z_i)$.
Given $G_i$, we generate the sample set $S_i = \{(x_{ij}, y_{ij}, z_{ij})\}_{j=1}^n$
according to the generative process described in Section~\ref{sec:tuebingen}.
Together with $S_i$, we annotate the triplet of labels $(l_{i1}, l_{i2}, l_{i3})$, where according to $G_i$,
\begin{itemize}
  \item $l_{i1} = +1$ if ``$X_i \to Y_i$'', $l_{i1} = -1$ if ``$X_i \leftarrow Y_i$'', and $l_{i1} = 0$ else.
  \item $l_{i2} = +1$ if ``$Y_i \to Z_i$'', $l_{i2} = -1$ if ``$Y_i \leftarrow Z_i$'', and $l_{i2} = 0$ else.
  \item $l_{i3} = +1$ if ``$X_i \to Z_i$'', $l_{i1} = -1$ if ``$X_i \leftarrow Z_i$'', and $l_{i1} = 0$ else.
\end{itemize}
Then, we add the following six elements to our training set:
\begin{align*}
(\{(x_{ij},y_{ij}, z_{ij})\}_{j=1}^n,+l_1), (\{(y_{ij},z_{ij}, x_{ij})\}_{j=1}^n,+l_2), (\{(x_{ij},z_{ij}, y_{ij})\}_{j=1}^n,+l_3),\\
(\{(y_{ij},x_{ij}, z_{ij})\}_{j=1}^n,-l_1), (\{(z_{ij},y_{ij}, x_{ij})\}_{j=1}^n,-l_2), (\{(z_{ij},x_{ij}, y_{ij})\}_{j=1}^n,-l_3),
\end{align*}
for all $1 \leq i \leq N$.  Therefore, our training set will consist on $6N$ sample sets and their paired labels. At this point, and given any sample $\{(u_{ij},
  v_{ij}, w_{ij})\}_{j=1}^n$ from the training set, we propose to use as
  feature vectors the concatenation of the $m-$dimensional empirical kernel
  mean embeddings \eqref{eq:meank3} of $\{u_{ij}\}_{j=1}^n$,
  $\{v_{ij}\}_{j=1}^n$, and $\{(u_{ij}, v_{ij}, w_{ij})\}_{j=1}^n$,
  respectively. 

\subsection{Test Phase}
To start, given $n_{te}$ test $d-$dimensional samples $S = \{(x_{1i}, \ldots,
x_{di})\}_{i=1}^{n_{te}}$, the hyper-parameters of the kernel and training data
synthesis process are transductively chosen, as described in
Section~\ref{sec:tuebingen}.

In order to estimate the causal graph underlying the test sample set $S$, we compute
three $d \times d$ matrices $M_\to$, $M_{\indep}$, and $M_{\leftarrow}$. Each of these
matrices will contain, at their coordinates $i,j$, the
probabilities of the labels ``$X_i \to X_j$'', ``$X_i \indep X_j$'', and ``$X_i
\leftarrow X_j$'', respectively, when averaged over all possible scalar confounders
$X_k$.  Using these matrices, we estimate the underlying causal graph by
selecting the type of each edge (forward, backward, or no edge) to be the one with
maximal probability according. As a post-processing step, we prune the least-confident
edges until the derived graph is acyclic. 

\end{document}